\documentclass{article}

\usepackage{colt11e}
\usepackage[round,comma]{natbib}
\usepackage{amsmath,amssymb,amsthm,amsfonts,mathtools,color,nicefrac}

\newcommand{\ignore}[1]{}

\newcommand{\EE}[1]{{{\mbox{\bf E}}\left[{#1}\right]}}
\newcommand{\Ep}[2]{{{\mbox{\bf E}}_{#1}\left[{#2}\right]}}

\newcommand{\abs}[1]{\left\lvert{#1}\right\rvert}
\newcommand{\norm}[1]{\left\lvert{#1}\right\rvert}
\newcommand{\infnorm}[1]{\norm{#1}_{\infty}}
\newcommand{\Norm}[1]{\left\lVert{#1}\right\rVert}
\newcommand{\trnorm}[1]{\Norm{#1}_{\Sigma}}
\newcommand{\maxnorm}[1]{\Norm{#1}_{\textrm{max}}}
\newcommand{\loss}{\textrm{loss}}

\newtheorem{theorem}{Theorem}
\newtheorem{lemma}{Lemma}

\newtheorem{definition}{Definition}

\newtheorem{remark}{Remark}

\def\N{\mathbf{N}}
\def\SS{\mathbb{S}}
\def\R{\mathbb{R}}
\def\diag{\mathrm{diag}}

\def\rank{\mathrm{rank}}
\def\X{\mathcal{X}}

\def\RR{\mathcal{R}}

\def\e{\epsilon}

\def\D{\mathcal{D}}

\def\H{\mathcal{H}}

\def\o{\omega}

\newcommand{\empRad}{\hat{\RR}}
\title{Concentration-Based Guarantees for Low-Rank Matrix Reconstruction}

\author{Rina Foygel\\
Department of Statistics\\University of Chicago \\
\texttt{\small rina@uchicago.edu}
\And 
Nathan Srebro\\
Toyota Technological Institute at Chicago\\
\texttt{\small nati@ttic.edu}
}

\begin{document}

\maketitle

\begin{abstract}
  We consider the problem of approximately reconstructing a partially-observed, 
  approximately low-rank matrix.  This problem has received much
  attention lately, mostly using the trace-norm as a surrogate to the
  rank.  Here we study low-rank matrix reconstruction using both the
  trace-norm, as well as the less-studied max-norm, and present
  reconstruction guarantees based on existing analysis on the
  Rademacher complexity of the unit balls of these norms. We show how
  these are superior in several ways to recently published guarantees
  based on specialized analysis.
\end{abstract}

\section{Introduction}

We consider the problem of (approximately) reconstructing an
(approximately) low-rank matrix based on observing a random subset of
entries.  That is, we observe $s$ randomly chosen entries of an
unknown matrix $Y\in\R^{n\times m}$, where we assume either $Y$ is of
rank at most $r$, or there exists $X\in\R^{n\times m}$ of rank at most
$r$ that is close to $Y$.  Based on these $s$ observations, we would
like to construct a matrix $\hat{X}$ that is as close as possible to
$Y$.

There has been much interest recently in computationally efficient
methods for reconstructing a partially-observed, possibly noisy,
low-rank matrix, and on accompanying guarantees on the quality of the
reconstruction and the required number of observations.  Since
directly searching for a low-rank matrix minimizing the empirical
reconstruction error is NP-hard \citep{Chistov}, most work has focused on
using the trace-norm (a.k.a. nuclear norm, or Schatten-1-norm) as a
surrogate for the rank.  The trace-norm of a matrix is the sum (i.e.
$\ell_1$-norm) of its singular values, and thus relaxing the rank
(i.e. the number of non-zero singular values) to the trace-norm is
akin to relaxing the sparsity of a vector to its $\ell_1$-norm, as is
frequently done in compressed sensing.  The analysis of the quality of
reconstruction has also been largely driven by ideas coming from
compressed sensing, typically studying the optimality conditions of
the empirical optimization problem, and often requiring various
``incoherence''-type assumptions on the underlying low-rank matrix.

In this paper we provide simple guarantees on approximate low-rank
matrix reconstruction using a different surrogate regularizer: the
$\gamma_{2:\ell_1\rightarrow\ell_{\infty}}$ norm, which we refer to
simply as the ``max-norm''.  This regularizer was first suggested by
\citet{MMMF}, though it has not received much attention since.  Here
we show how this regularizer can yield guarantees that are superior in
some ways to recent state-of-the-art.  In particular, we show that
when the entries are uniformly bounded, i.e.
$\infnorm{X}=\mathbf{O}(1)$ (this corresponds to the ``no spikiness''
assumption of \citet{NW}, and is also assumed by \citet{K} and in the
approximate reconstruction guarantee of \citet{KMO}), then the
max-norm regularized predictor requires a sample size of
\begin{equation}
  \label{eq:introsampsize}
  s = \mathbf{O}\left(\frac{r(n+m)}{\e} \cdot\frac{\sigma^2+\epsilon}{\epsilon}\cdot  \log^3(1/\epsilon) \right)
\end{equation}
to achieve mean-squared reconstruction error
$\frac{1}{nm}\lvert{\hat{X}-Y}\rvert^2_2=\sigma^2+\epsilon$, where
$\sigma^2$ is the the mean-squared-error of
the best rank-$r$ approximation of $Y$---that is, $\sigma^2=\frac{1}{nm}|X-Y|^2_2$, where $X$ is 
the rank-$r$ approximation.  When $Y$ is exactly low-rank
(the noiseless case), $\sigma^2=0$ and the sample complexity is
$\mathbf{O}\left(\frac{r(n+m)}{\e}\cdot\log^3(1/\epsilon)\right)$.  Compared to the three recent
similar bounds mentioned above, this guarantee avoids the extra
logarithmic dependence on the dimensionality, as well as the
assumption of independent noise, but has a slightly worse dependence
on $\epsilon$.  We emphasize that we do not make any assumptions about
the noise, nor about incoherence properties of the underlying low-rank
matrix $X$.

We also provide a guarantee on the mean-absolute-error of the
reconstruction, and discuss guarantees for reconstruction using the
trace-norm as a surrogate.  Using the trace-norm allows us to provide
mean-absolute-error guarantees also for matrices where the magnitudes
are {\em not} uniformly bounded (i.e.~``spiky'' matrices).  We further
show that a spikiness assumption is necessary for squared-error
approximate reconstruction of low-rank matrices, regardless of the
estimator used.

Instead of focusing on optimality conditions as in previous work, our
guarantees follow from generic generalization guarantees based on the
Rademacher complexity, and an analysis of the Rademacher complexity of
the max-norm and trace-norm balls conducted by
\citet{ShraibmanSrebro}.  To obtain the desired low rank
reconstruction guarantees, we combine these with bounds on the
max-norm and trace-norm in terms of the rank.  The point we make here
is that these fairly simple arguments, mostly based on the work of
\citet{ShraibmanSrebro}, are enough to obtain guarantees that are in
many ways better and more general than those presented in recent
years.

\paragraph{Notation.}
We use $\norm{M}$ to denote the elementwise norms of a matrix $M$:
$\norm{M}_1=\sum_{ij}|M_{ij}|$, $\norm{M}_2$ is the Frobenius norm,
and $|M|_{\infty}=\max_{ij}|M_{ij}|$.  We discuss $n \times m$
matrices, and without loss of generality always assume $n\geq m$.

\section{The Max-Norm and Trace-Norm}

We will consider the following two matrix norms, which are both
surrogates for the rank:

\begin{definition}\label{def_trace} The {\bf trace-norm} of a matrix
  $X\in\R^{n\times m}$ is given by:
$$\trnorm{X}=\sum (\text{singular values of $X$})=\min_{U,V:X=UV^T}|U|_2|V|_2\;.$$
\end{definition}

\begin{definition}\label{def_max} The {\bf max-norm} of a matrix
  $X\in\R^{n\times m}$ is given by:
$$\maxnorm{X}=\min_{U,V:X=UV^T}\left(\max_i |U_{(i)}|_2\right)\left(\max_j|V_{(j)}|_2\right)\;,$$
where $U_{(i)}$ and $V_{(j)}$ denote the $i^{\mathrm{th}}$ row of $U$ and the $j^{\mathrm{th}}$ row of $V$, respectively.
\end{definition}

Both the trace-norm and the max-norm are semi-definite representable
\citep{FazelHindiBoyd,MMMF}.  Consequently, optimization problems involving a
constraint on the trace-norm or max-norm, a linear or quadratic
objective, and possibly additional linear constraints, are solvable
using semi-definite programming.  We will consider estimators which
are solutions to such problems.  \ignore{See \citet{} and references therein
for practical approaches for efficiently solving such semi-definite
programs.}

\citet{ShraibmanSrebro} and later \citet{Sherstov} studied the
max-norm and trace-norm as surrogates for the rank in a classification
setting, where one is only concerned with the signs of the underlying
matrix.  They showed that a sign matrix might be realizable with low
rank, but realizing it with unit margin might require exponentially
high max-norm or trace-norm.  Based on this analysis, they argued that
the max-norm and trace-norm {\em cannot} be used to obtain
reconstruction guarantees for sign matrices of low rank matrices.

Here, we show that in a regression setting, the situation is quite
different, and the max-norm and trace-norm {\em are} good convex
surrogates for the rank.  The specific relationship between these
surrogates and the rank is determined by how we control the scale of
the matrix $X$ (i.e.~the magnitude of its entries).  This will be made
explicit in the next section, but for now we state the bounds on the
trace-norm and max-norm in terms of the rank which we will leverage in
Section \ref{sec:mainresults}.

By bounding the $\ell_1$ norm of the singular values (i.e. the
trace-norm) by their $\ell_2$ norm (i.e.~the Frobenius norm) and the
number of non-zero values (the rank), we obtain the following
relationship between the trace-norm and Frobenius norm:
\begin{equation}\label{eqn_trace_bound}
|X|_2\leq \|X\|_{\Sigma}\leq \sqrt{\rank(X)}\cdot|X|_2\;.
\end{equation}
Interpreting the Frobenius norm as specifying the {\em average} entry
magnitude, $\frac{1}{nm}\norm{X}^2_2$, we can view the above as upper
bounding the trace-norm with the square root of the rank, when the average entry
magnitude is fixed.

An analagous bound for the max norm, substituting $\ell_{\infty}$ norm
(maximal entry magnitude) for Frobenius norm (average entry
magnitude), can be obtained as follows:
\begin{lemma}\label{lem:max_vs_infinity}For any $X\in\R^{n\times m}$,
$|X|_{\infty}\leq \|X\|_{\max}\leq \sqrt{\rank(X)}\cdot|X|_{\infty}$.
\end{lemma}
\begin{proof}
Consider the minimizing factorization $X=UV^{T}$ and let $X_{ij}$ be
the largest magnitude entry in $X$, then:
$\maxnorm{X}\geq\norm{U_{(i)}}\cdot\norm{V_{(j)}}\geq\abs{X_{ij}}=\infnorm{X}$.  

To obtain the upper bound we first write the max-norm as
\citep{LeeShraibmanSpalek}:
\begin{align}
  \|X\|_{\max}
  &=\sup_{p,q}\|\diag(p)X\diag(q)^2\|_{\Sigma}\;, \\
\intertext{where the supremum is over nonnegative unit vectors $p,q$.
  We can now continue using \eqref{eqn_trace_bound}:}
&\leq \sup_{p,q}
\sqrt{\rank(\diag(p)X\diag(q))}\cdot |\diag(p)X\diag(q)|_2 \notag \\
&\leq \sup_{p,q} \sqrt{\rank X} \cdot
\sqrt{\sum_{ij}p_i^2q_j^2X_{ij}^2} =\sqrt{\rank X}\infnorm{X}\;. \notag \qedhere
\end{align}
\end{proof}

\section{Reconstruction Guarantees}\label{sec:mainresults}

The theorems below provide reconstructions guarantees, first under the
a mean-absolute-error reconstruction measure (Theorem \ref{L1bound})
and then under a mean-squared-error reconstruction measure (Theorem
\ref{L2bound}).  Since the guarantees are for {\em approximate}
reconstruction, we must impose some notion of scale.  In other words,
we can think of measuring the error relative to the scale of the
data---if $Y$ is multiplied by some constant, then obviously the
reconstruction error would also be multiplied by this constant.  In
the theorems below we refer to two notions of scale: the {\em average}
squared magnitude of matrix entries, i.e.~$\frac{1}{nm}\norm{X}^2_2$,
and the {\em maximal} magnitude of matrix entries, i.e.~$\infnorm{X}$.
For simplicity and without loss of generality, the results are stated
for unit scale.

An issue to take note of is whether the $s$ observed entries of $Y$
are chosen with or without replacement, i.e.~whether we choose a set
$S$ of entries uniformly at random over all sets of exactly $s$
entries (no replacements), or whether we make $s$ independent uniform
choices of entries, possibly observing the same entry twice.  Our
results apply in both cases.

\begin{theorem}\label{L1bound}
For any $M,Y\in\R^{n\times m}$ where $M$ is of rank at most $r$:
\begin{itemize}
\item[a.] {\bf Entry magnitudes bounded on-average.} Consider the
  estimator\footnote{If $S$ is chosen with replacements, it is a
    multiset, and the summation $\sum_{(i,j)\in S}$ should be
    interpreted as summation with repetitions.}
$$\hat{X}(S)=\arg\min_{\trnorm{X}\leq \sqrt{rnm}}\sum_{(i,j)\in S}|Y_{ij}-X_{ij}|\;.$$
If $\frac{1}{nm}|M|^2_2\leq 1$ and 
$s\geq \mathbf{O}\left(\frac{ r(n+m)\log(n)}{\e^{2}}\right)$,
then in expectation over a sample $S$ chosen either uniformly over
sets of size $s$ (without replacements) or by choosing $s$ entries
uniformly and independently (with replacements):
$$\frac{1}{nm}|Y-\hat{X}(S)|_1\leq \frac{1}{nm}|Y-M|_1+\epsilon\;.$$
\item[b.] {\bf Entry magnitudes bounded uniformly.} Consider the estimator
$$\hat{X}(S)=\arg\min_{\|X\|_{\max}\leq \sqrt{r}}\sum_{(i,j)\in
  S}|Y_{ij}-X_{ij}|\;.$$ If $\infnorm{M}\leq 1$ and
$s\geq \mathbf{O}\left(\frac{ r(n+m)}{\e^{2}}\right)$, then in expectation over
a sample $S$ of size $s$ chosen either with or without replacements as
above:
$$\frac{1}{nm}|Y-\hat{X}(S)|_1\leq \frac{1}{nm}|Y-M|_1+\e\;.$$
\end{itemize}
\end{theorem}

\begin{remark}\label{rem:L1_whp}
The above results can also be shown to hold in high probability over the sample $S$, rather than in expectation. Specifically, to ensure that the results of Theorem~\ref{L1bound} hold with probability at least $1-n^{-\beta}$ (for sampling with replacement) or $1-n^{-(\beta-2)}$ (for sampling without replacement), it is sufficient to change the sample size requirement to $s\geq\mathbf{O}\left(\frac{r(n+m)\log (n)+\beta\log(n)}{\e^2}\right)$ (in the trace-norm case) or $s\geq\mathbf{O}\left(\frac{r(n+m)+\beta\log(n)}{\e^2}\right)$ (in the max-norm case).
\end{remark}

\begin{theorem}\label{L2bound}
  For any $Y=M+Z\in\R^{n\times m}$ where $|Z|_{\infty}\leq
  \sqrt{\frac{rn}{ \log n}}$ and $M$ is of rank at most $r$ with
  $\infnorm{M}\leq 1$, denote $\sigma^2 = \frac{1}{nm}\norm{Z}^2_2$.  Consider the estimator
\begin{equation}\label{eq:XhatL2}
\hat{X}(S)=\arg\min_{\|X\|_{\max}\leq \sqrt{r}}\sum_{(i,j)\in S}(Y_{ij}-X_{ij})^2\;.
\end{equation}
If $s\geq
\mathbf{O}\left(\frac{r(n+m)}{\e}\cdot\frac{\sigma^2+\e}{\e}\cdot(\log^3(r/\e)+\beta)\right)$,
then, with probability at least $1-n^{-\beta}$ over a sample $S$ of
size $s$ chosen with replacement, or with probability at least $1-n^{-(\beta-2)}$ over a sample $S$ of size $s$ chosen without replacement,
\begin{equation}
  \label{eq:l2errbound}
\frac{1}{nm}|Y-\hat{X}(S)|^2_2\leq \sigma^2 + \e\;.  
\end{equation}
If we instead use the estimator:
\begin{equation}\label{eq:sqmaxestinf}
\hat{X}(S)=\arg\min_{\substack{\|X\|_{\max}\leq \sqrt{r}\\
    |X|_{\infty}\leq 1}}\sum_{(i,j)\in S}(Y_{ij}-X_{ij})^2\;,
\end{equation}
then we obtain \eqref{eq:l2errbound} when $s\geq\mathbf{O}\left(\frac{r(n+m)}{\e}\cdot\frac{\sigma^2+\e}{\e}\cdot(\log^3(1/\e)+\beta)\right)$.
\end{theorem}
The estimator \eqref{eq:sqmaxestinf} is SDP-representable, though
potentially more cumbersome.
\begin{remark}\label{rem:noisebound}
  The requirement on the maximal magnitude of the error in Theorem
  \ref{L2bound}, $|Z|_{\infty}\leq \sqrt{\frac{rn}{ \log n}}$, is very
  generous, and easily holds with high probability for sub-exponential
  noise.  A stricter requirement, e.g. $\mathbf{O}(\sqrt{r \log n})$, which still holds with high probability for
  subgaussian  noise, yields a guarantee
  with exponentially high probability $1-e^{-n/\log n}$, without a sample-complexity dependence on $\beta$.
\end{remark}
\begin{remark}\label{rem:L2result} 
  A guarantee similar to Theorem \ref{L2bound} can be obtained if we
  can ensure $\maxnorm{M}\leq A$, for some $A$, without requiring
  $\infnorm{M}\leq 1$.  For $\hat{X}(S)=\arg\min_{\maxnorm{X}\leq
    A}\sum_{ij\in S}(Y_{ij}-X_{ij})^2$, we have \eqref{eq:l2errbound}
  with a sample of size
$$s\geq\mathbf{O}\left(\frac{A^2(n+m)}{\e}\cdot\frac{\sigma^2+\e}{\e}\cdot(\log^3(A^2/\e)+\beta)\right)\;.$$
In Section \ref{sec:ExactRec:MaxMu}, we will see how certain incoherence assumptions
used in previous bounds yield a bound on $\maxnorm{M}$, and compare
the max-norm based reconstruction guarantee to the previously
published results.
\end{remark}

  In Theorems \ref{L1bound} and \ref{L2bound} we do not assume the
  noise, i.e. the entries of $Z=Y-M$, are independent or
  zero-mean---in fact, we make no assumption on $Z$, other than the
 very generous upper bound $|Z|_{\infty}\leq \sqrt{\frac{rn}{ \log
      n}}$ discussed above.  When entries of $Z$ can be arbitrary, it
  is not possible to ensure reconstruction of $M$ (e.g. we can set
  things up so $Y$ actually has lower rank then $M$, and so it is
  impossible to identify $M$).  Consequently, in Theorems
  \ref{L1bound} and \ref{L2bound} we instead bound the excess error in
  predicting $Y$ itself.  If entries of $Z$ {\em are} independent and
  zero-mean, then we may give the following guarantee about reconstructing the underlying matrix $M$:
  
\begin{theorem}\label{L2bound_inderror}
  For $(i,j)\in[n]\times[m]$, let $\mathcal{F}_{(i,j)}$ be any
  mean-zero distribution. Suppose that the observed entries of $Y$ are
  given by $Y_{(i_t,j_t)}=M_{(i_t,j_t)}+Z_{t}$ for $t=1,2,\dots,s$,
  where $(i_t,j_t)\stackrel{iid}{\sim}Unif([n]\times[m])$ and
  $Z_t|(i_t,j_t) \sim \mathcal{F}_{(i_t,j_t)}$ independently for each
  $t$.  That is, the noise is independent and zero-mean (though its distribution is
  allowed to depend on the location of the observation), the sample is
  drawn with replacement, and if an entry of the matrix is observed
  more than once, then the noise on the entry is drawn independently
  each time.

Assume $|M|_{\infty}\leq 1$, $\rank(M)\leq r$, and $\sup_{t\in[s]}|Z_{t}|\leq \mathbf{o}\left(\sqrt{\frac{rn}{\log n}}\right)$ with high probability. Denote
$$\sigma^2=\frac{1}{nm}\sum_{i,j}E_{Z_{ij}\sim \mathcal{F}_{ij}}(Z_{ij}^2)\;.$$
For the estimator given in Equation \eqref{eq:XhatL2}, with high
probability over the sample $S$ of size
$s\geq\mathbf{O}\left(\frac{r(n+m)}{\e}\cdot\frac{\sigma^2+\e}{\e}\cdot \log^3(r/\e)\right)$,
\begin{equation}
  \label{eq:l2errbound_inderror}
\frac{1}{nm}|M-\hat{X}(S)|^2_2\leq\e\;.  
\end{equation}
Alternatively, is $S$ is sampled uniformly without replacements, with
the same assumptions and sample size, and as long as $s\leq \tfrac{K+1}{e}
(nm)^{1-\tfrac{1}{K+1}}$, we have $\frac{1}{nm}|M-\hat{X}(S)|^2_2\leq
4K\e$.
\end{theorem}

\begin{remark}\label{rem:ind_error_log_term}
When sampling without replacement, we imposed both a lower bound and
an upper bound on the sample size.  For these two bounds to be
compatible (in an asymptotic sense) for a fixed $K$, we need
$m=\Omega\left(n^a\right)$ for some positive power $a$, and make
$\epsilon$ arbitrarily small.  Alternately, we can set
$K=\mathbf{O}(\log n)$, ensuring the upper bound on $s$ always holds
(since $s\leq nm$ necessarily), yielding
$\frac{1}{nm}|M-\hat{X}(S)|^2_2\leq\e$ whenever
$s\geq\mathbf{O}\left(\frac{r(n+m)\log(n)}{\e}\cdot\frac{\sigma^2+\e}{\e}\cdot \log^3(r/\e)\right)$.
\end{remark}

The remainder of this Section is organized as follows: In
Section~\ref{proof:withrepl}, we prove Theorems~\ref{L1bound} and
~\ref{L2bound} in the case where the sample is drawn without
replacement.  In Section~\ref{proof:tracenorm}, we discuss possible
bounds of the mean-squared-error, as in Theorem~\ref{L2bound}, but
using the trace-norm.  In Section~\ref{proof:withoutrepl}, we compare
sampling with and without replacement, establishing
Theorems~\ref{L1bound} and~\ref{L2bound} also for sampling with
replacement.  In Section~\ref{proof:ind_error}, we turn to the setting
of independent mean-zero noise, and prove
Theorem~\ref{L2bound_inderror} in both the sampling-with-replacement
and sampling-without-replacement settings.

\subsection{Proof of Theorems \ref{L1bound} and \ref{L2bound} when $S$ is drawn with replacement}\label{proof:withrepl}

We first establish the Theorems for a sample chosen i.i.d.~with
replacements.  In this case, following \citet{ShraibmanSrebro}, we may
view matrix reconstruction as a prediction problem, by regarding a
matrix $X\in\mathbb{R}^{n\times m}$ as a function
$[n]\times[m]\rightarrow\mathbb{R}$. Each observation in the training
set consists of a covariate $(i,j)\in[n]\times[m]$ and an observed
noisy response $Y_{ij}\in\mathbb{R}$.  Here, we assume that the
distribution over $[n]\times[m]$ is uniform, and the joint
distribution over $(i,j)$ and its response is determined by the
unknown $Y$.  The hypothesis class is then a set of matrices bounded
in either trace-norm or max-norm, and for a particular hypothesis
$X\in\mathbb{R}^{n\times m}$, the averaged error $\frac{1}{nm}|Y-X|_1$
or $\frac{1}{nm}|Y-X|^2_2$ is equal to the expected loss
$L(X)=\Ep{ij}{\loss(X_{ij},Y_{ij})}$ under either the absolute-error
or squared-error loss, respectively.

\citet{ShraibmanSrebro} established bounds on the Rademacher
complexity of the trace-norm and max-norm balls.  For any sample of
size $s$, the empirical Rademacher complexity of the max-norm ball is
bounded by
\begin{equation}
  \label{eq:maxRad}
  \empRad_s\left(\left\{ X \in \R^{n\times m} \;\middle|\;
      \maxnorm{X}\leq A\right\}\right) \leq 12 \sqrt{\frac{A^2(n+m)}{s}}\;.
\end{equation}
Although the empirical Rademacher complexity of the trace-norm ball
might be fairly high, the {\em expected} Rademacher complexity, for a
random sample of $s$ independent {\em uniformly} chosen index pairs (with
replacements) can be bounded as
\begin{equation}
  \label{eq:trRad}
  \EE{\empRad_s\left(\left\{ X \in \R^{n\times m} \;\middle|\;
      \trnorm{X}\leq A\right\}\right)} \leq K \sqrt{\frac{\frac{A^2}{nm}(n+m)\log(n)}{s}}
\end{equation}
for some numeric constant $K$ (this is a slightly better
bound then the one given by \citet{ShraibmanSrebro}, and is proved in
Appendix \ref{app:trnormrad}).

Since the absolute error loss, $\loss(x,y)=\abs{x-y}$, is 1-Lipschitz,
these Rademacher complexity bounds immediately imply \citep{BartlettMendelson}:
\begin{equation}
  \label{eq:maxL1gen}
  \frac{1}{nm}\norm{Y-\hat{X}(S)}_1 \leq \inf_{\maxnorm{X}\leq A} \left( \frac{1}{nm}\norm{Y-X}_1\right) + 24 \sqrt{\frac{A^2(n+m)}{s}}
\end{equation}
for $\hat{X}(S) = \arg\min_{\maxnorm{X}\leq A} \sum_{(i,j)\in
  S}|Y_{ij}-X_{ij}|$, and:
\begin{equation}
  \label{eq:trL1gen}
  \frac{1}{nm}\norm{Y-\hat{X}(S)}_1 \leq \inf_{\trnorm{X}\leq A}\left(\frac{1}{nm} \norm{Y-X}_1\right)+ 2K
  \sqrt{\frac{\frac{A^2}{nm}(n+m)\log(n)}{s}}
\end{equation}
for $\hat{X}(S) = \arg\min_{\trnorm{X}\leq A} \sum_{(i,j)\in
  S}|Y_{ij}-X_{ij}|$.    (For details, see Lemma~\ref{lem:BM_radbound} in Appendix~\ref{app:BM_radbound}.)  These provide guarantees on reconstructing
matrices with bounded max-norm or trace-norm.  Choosing $A=\sqrt{r}$
for the max-norm and $A=\sqrt{rnm}$ for the trace-norm, Theorem
\ref{L1bound} (for sampling with replacement) follows from
Equation~(\ref{eqn_trace_bound}) and Lemma~\ref{lem:max_vs_infinity}. (Remark~\ref{rem:L1_whp}
follows from the results of~\citet{BartlettMendelson} with identical arguments for the sampling-with-replacement case.)

In order to obtain Theorem \ref{L2bound}, we use a recent bound on the
excess error with respect to a {\em smooth} (rather then Lipschitz) loss
function, such as the squared loss.  Specifically, Theorem 1 of
\citet{SST} states that, for a class of predictors
$X:\mathcal{I}\rightarrow [-B,B]$ and a loss function bonded by $b$
with second derivative bounded by $H$, with probability at least
$1-\delta$ over a random sample of size $s$,
\begin{align}
  \label{eq:smoothbound}
  &L(\hat{X}) \leq L^* + O\left(\sqrt{L^* \tilde{\RR}_s }+
      \tilde{\RR}_s \right)\;, \\
&\quad L^* = \inf_X L(X)\;, \notag \\
&\quad \tilde{\RR}_s = H \RR_s^2 \log^3 \left(\frac{B}{\RR_s}\right)+\frac{b\log(\log(s)/\delta)}{s}\;,
\end{align}
where the infimum is over predictors in the class, $\hat{X}$ is the
empirical error minimizer in the class, and $\RR_s$ is an upper bound
on the Rademacher complexity for all samples of size $s$.

In our case, for the class $\{ X | \maxnorm{X}\leq A \}$ and the
squared loss, we have $B=\sup_X \sup_{ij} |X_{ij}| = \sup_X\infnorm{X} \leq
\sup_X\maxnorm{X} \leq A$ and $b=\sup_X \infnorm{X-Y}^2 \leq
\sqrt{\frac{4A^2 (n+m)}{\log(n+m)}}$, when we assume $|Z|_{\infty}\leq A\sqrt{\frac{n+m}{\log(n+m)}}$.  Applying the bound \eqref{eq:maxRad} on the Rademacher complexity yields:
\begin{align}
\tilde{\RR}_s
&= \mathbf{O}\left(\frac{A^2(n+m)}{s}\log^3\left(\frac{s}{n}\right)+\frac{A^2(n+m)\log\log s}{s\log (n+m)}+\frac{A^2(n+m)\log(1/\delta)}{s\log n}\right)\\
&=
\mathbf{O}\left(\frac{A^2(n+m)}{s}\left(\log^3\left(\frac{s}{n+m}\right)+\frac{\log(1/\delta)}{\log
      n}\right)\right)\;.
\end{align}
Here the last inequality uses the fact that $s\leq n^2$, while the
next-to-last inequality assumes $s\geq e^3(n+m)$, and applies the fact
that $x^2\log^3(1/x)$ is an increasing function for $x<e^{-1.5}$,
where in this case $x=\sqrt{\frac{n+m}{s}}$.  

Remark~\ref{rem:L2result} follows immediately. The first claim in Theorem~\ref{L2bound} follows when we assume $|M|_{\infty}\leq 1$ and $\rank(M)\leq r$ and set $A=\sqrt{r}$ (since, by Lemma~\ref{lem:max_vs_infinity}, $\|M\|_{\max}\leq A$). If we instead consider the class $\{X:\|X\|_{\max}\leq \sqrt{r},|X|_{\infty}\leq 1\}$, then in the notation of~(\ref{eq:smoothbound}), we may define $B=1$ instead of $B=A=\sqrt{r}$, and thus obtain
 \begin{align}
\tilde{\RR}_s
&=
\mathbf{O}\left(\frac{r(n+m)}{s}\left(\log^3\left(\frac{s}{r(n+m)}\right)+\frac{\log(1/\delta)}{\log
      n}\right)\right)\;,
\end{align}
which yields the second claim of Theorem~\ref{L2bound}.

Finally, we prove the claim Remark~\ref{rem:noisebound}. If instead we assume $|Z|_{\infty}\leq \sqrt{r\log n}$, then in the the notation of~(\ref{eq:smoothbound}), we may define $b=r\log n$ instead of $b=\frac{4A^2n}{\log n}=\frac{4rn}{\log n}$, and thus obtain
\begin{align}
\tilde{\RR}_s
&=
\mathbf{O}\left(\frac{r(n+m)}{s}\left(\log^3\left(\frac{s}{(n+m)}\right)+\frac{\log n\cdot \log(1/\delta)}{n+m}\right)\right)\;.
\end{align}
For $\delta\leq e^{-n/log n}$, the second term is dominated by the first; therefore the sample complexity no longer depends on $\beta$.

\subsection{Bounds on $\ell_2$ error using the trace norm}\label{proof:tracenorm}

In Theorem \ref{L1bound}, we saw that for mean-absolute-error
matrix reconstruction, using the trace-norm instead of the max-norm allows us
to forgo a bound on the spikiness, and rely only on the average
squared magnitude $\frac{1}{nm}\norm{Y}_2^2$.  One might hope that we can
similarly get a squared-error reconstruction guarantee using the
trace-norm and without a spikiness bound that was required in Theorem
\ref{L2bound}.  Unfortunately, this is not possible.

In fact, as the following example demonstrates, it is not possible to
reconstruct a low-rank matrix to within much-better-then-trivial
squared-error without a spikiness assumption, and relying only on
$\frac{1}{nm}\norm{Y}_2\leq 1$.  Specifically, consider an $n \times
m$ matrix
$$Y = \sqrt{m/r} \left(\, A \, | \, 0_{n \times (m-r)} \right)$$
where $A \in \{\pm 1\}^{n \times r}$ is an arbitrary sign matrix.  The
matrix $Y$ has rank at most $r$ and average squared magnitude
$\frac{1}{nm} \norm{Y}^2_2 = 1$ (but maximal squared magnitude
$\norm{Y}^2_{\infty} = m/r$).  Now, with even half the entries
observed (i.e. $s=nm/2$), we have no way of reconstructing the
unobserved entries of $A$, as any values we choose for these entries
would be consistent with the rank-$r$ assumption, yielding an expected
average squared error of at least $1/2$.  We can conclude that
regardless of the estimator, controlling the average squared magnitude
is not enough here, and we cannot expect to obtain a squared-error
reconstruction guarantee based on $\frac{1}{nm}\norm{Y}^2_2$, even if
we use the trace-norm.  

We note that if $\infnorm{M},\infnorm{Y}=\mathbf{O}(1)$, then the
squared-loss in the relevant regime has a bounded Lipschitz constants,
and Theorem~\ref{L1bound}a applies. In particular, if
$\infnorm{M},\infnorm{Y}\leq 1$, then we can consider the estimator
\begin{equation}\label{eq:sqmaxestinftrnorm}
\hat{X}(S)=\arg\min_{\substack{\trnorm{X}\leq \sqrt{rnm}\\
    |X|_{\infty}\leq 1}}\sum_{(i,j)\in S}(Y_{ij}-X_{ij})^2\;.
\end{equation}
Since we now only need to consider $X$ where $\abs{X_{ij}-Y_{ij}}\leq
2$, the squared-loss in the relevant domain is 4-Lipschitz.  We can
therefore use the standard generalization results for Lipschitz loss as
in Theorem \ref{L1bound}, and obtain that with high probability over a
sample of size
\begin{equation}\label{eq:unsatisfying}
s\geq \mathbf{O}\left(\frac{r(n+m)\log
    n}{\e^2}\right)\;,
\end{equation}
we have $\frac{1}{nm}|Y-\hat{X}(S)|^2_2\leq \sigma^2+\epsilon$. However, this result gives
a dependence on $\epsilon$ that is quadratic, as opposed to the more favorable dependence (at least when
$\epsilon = \Omega(\sigma^2)$) of Theorem \ref{L2bound}.

We believe that, when $|M|_{\infty},|Y|_{\infty}\leq \mathbf{O}(1)$,
it is possible to improve the dependence on $\e$ to a dependence
similar to that of Theorem~\ref{L2bound} (this would require a more
delicate analysis then that of \citet{SST}, as their techniques rely
on bounding the worst-case Rademacher complexity).  But even this
would not give any advantage over the max-norm, since the bound on
$|M|_{\infty}$ could not be relaxed, while an additional
factor of $\log n$ would be introduced into the sample complexity
(coming from the Rademacher complexity calculation for the
trace-norm).  It seems then, that at least in terms of the quantities
and conditions considered in this paper, as well as elsewhere in the
low-rank reconstruction literature we are familiar with, there is no
theoretical advantage for the trace-norm over the max-norm in terms of
squared-error approximate reconstruction, though there could be an
advantage for the max-norm in avoiding a logarithmic factor.

\subsection{Sampling with or without replacement in Theorems~\ref{L1bound} and~\ref{L2bound}}\label{proof:withoutrepl}

Theorems~\ref{L1bound} and~\ref{L2bound} give results that hold for
either sampling with replacement or sampling without replacement. When
an entry of the matrix $Y$ is sampled twice, the same value is
observed each time---no new information about the matrix is observed,
and so intuitively, sampling without replacement should yield strictly
better results than sampling with replacement.  The two lemmas below,
proved in the Appendix, establish that sampling without
replacement is indeed as at least as good as sampling with
replacement  (up to a constant).

Before stating the lemmas, we briefly introduce some notation. Let
$L(X)$ denote the loss for an estimated matrix $X$; that is,
$L(X)=\frac{1}{nm}\norm{Y-X}_1$ or $L(X)=\frac{1}{nm}\norm{Y-X}^2_2$,
as appropriate. Let $\hat{L}_S(X)$ denote the empirical loss,
$\hat{L}_S(X)=\sum_{(i,j)\in S}|Y_{ij}-X_{ij}|^p$ (where $p\in\{1,2\}$
and the sum includes repeated elements in $S$). Let $\mathcal{D}^s$
and $\mathcal{D}^s_{w/o}$ denote the distributions of a sample of size
$s$ drawn uniformly at random from the matrix, either with or without
replacement, respectively.

\begin{lemma}\label{lem:w_or_wo:expectation} Let $\X$ denote any class of matrices, with $\D^s$ and $\D^s_{w/o}$ defined as above. Then
$$E_{S\sim \mathcal{D}^s_{w/o}}\left[\sup_{X\in\X}L(X)-\hat{L}_S(X)\right]\leq E_{S\sim\D^s}\left[\sup_{X\in\X}L(X)-\hat{L}_S(X)\right]\;.$$
\end{lemma}

\begin{lemma}\label{lem:w_or_wo:prob}  Let $\X$ denote any class of matrices, with $\D^s$ and $\D^s_{w/o}$ defined as above. Then for any $c\in\mathbb{R}$, and for any function $g$,
$$P_{S\sim \mathcal{D}^s_{w/o}}\left\{\left(\sup_{X\in\X}g(L(X))-\hat{L}_S(X)\right)\geq c\right\}\leq 4s\cdot P_{S\sim \mathcal{D}^s}\left\{\left(\sup_{X\in\X}g(L(X))-\hat{L}_S(X)\right)\geq c\right\}\;.$$
\end{lemma}
 
For the $\ell_1$-loss case, the Rademacher bounds~(\ref{eq:maxL1gen})
and~(\ref{eq:trL1gen}) are derived from \cite{BartlettMendelson} by
bounding
$E_{S\sim\mathcal{D}^s}\left(\sup_{X\in\X}L(X)-\hat{L}_S(X)\right)$ (or $P_{S\sim\mathcal{D}^s}\left(\sup_{X\in\X}L(X)-\hat{L}_S(X)\geq c\right)$, for the proof of Remark~\ref{rem:L1_whp}).
By Lemma~\ref{lem:w_or_wo:expectation}, the same bound then holds for
the same expectation taken over $S\sim\mathcal{D}^s_{w/o}$, and
therefore~(\ref{eq:maxL1gen}) and~(\ref{eq:trL1gen}) must hold for
this case as well. This implies that the results of
Theorem~\ref{L1bound} (and Remark~\ref{rem:L1_whp}) hold for sampling without replacement as well as
sampling with replacement.

Similarly, for the $\ell_2$-loss case, the Rademacher
bound~(\ref{eq:smoothbound}) is derived in \cite{SST} by bounding
$\sup_{X\in\X} L(X)-\sqrt{a\cdot L(X)}-\hat{L}_S(X)$ for some constant
$a$, with probability at least $1-\delta$ over $S\sim\mathcal{D}^s$.
Defining $g(L)=L-\sqrt{a\cdot L}$, the same bound must therefore hold
with probability at least $1-4s\delta\geq 1-4n^2\delta$ over
$S\sim\mathcal{D}^s_{w/o}$, and therefore~(\ref{eq:smoothbound}) holds
for this case also. This implies that the results of
Theorem~\ref{L2bound} (and the subsequent remarks) hold for sampling
without replacement as well as sampling with replacement.

\subsection{Proof of Theorem~\ref{L2bound_inderror}: independent errors in the $\ell_2$-loss setting.}\label{proof:ind_error}

\ignore{For $(i,j)\in[n]\times[m]$, let $\mathcal{F}_{(i,j)}$ be any mean-zero distribution. Suppose that the observed entries of $Y$ are given by $Y_{(i_t,j_t)}=M_{(i_t,j_t)}+Z_{t}$ for $t=1,2,\dots,s$, where $(i_t,j_t)\stackrel{iid}{\sim}Unif([n]\times[m])$ and the noise $Z_t$ s drawn independently from $\mathcal{F}_{(i_t,j_t)}$. 
Assume $|M|_{\infty}\leq 1$, $\rank(M)\leq r$, and the distributions $\mathcal{F}_{(i,j)}$ are such that $\sup_{t\in[s]}|Z_{t}|\leq \mathbf{o}\left(\sqrt{\frac{rn}{\log n}}\right)$ with high probability.  In this setting, we aim to recover $M$ rather than $Y$, and achieve a meaningfully stronger result when sampling with replacement. This is due to the fact that when noise is generated independently at each observation, it is no longer the case that sampling without replacement is intuitively as good as sampling with replacement---although more entries are observed when sampling without replacement, the tradeoff is that some of the observed entries have more accurate estimates when sampling with replacement due to the independently generated noise on each observation.}

First, we prove the theorem when sampling with replacement. For a matrix $X$, let $L(X)$ denote the expected squared error for a randomly sampled entry, that is,
$$L(X)=\frac{1}{nm}\sum_{(i,j)} E((Y_{ij}-X_{ij})^2)=\frac{1}{nm}\sum_{(i,j)}E_{Z\sim\mathcal{F}_{(i,j)}}((Z+M_{ij}-X_{ij})^2)\;.$$
Now write $\sigma^2=\frac{1}{nm}\sum_{(i,j)}E_{Z\sim\mathcal{F}_{(i,j)}}(Z^2)$. Then $L(M)=\sigma^2$.

Then, for any sample $S$, given $\hat{X}(S)$ which is a random matrix depending on some observed sample, the expected loss (over a future observation of an entry in the matrix) of $\hat{X}(S)$ satisfies the following (due to the fact that noise in a future observation of the matrix has zero mean and is independent from $\hat{X}(S)$):
\begin{align*}
L(\hat{X}(S))
&=E_{(i,j)}\left((Y_{ij}-\hat{X}(S)_{ij})^2\Big\vert\hat{X}(S)\right)
=E_{(i,j), Z\sim \mathcal{F}_{ij}}\left((Z+M_{ij}-\hat{X}(S)_{ij})^2\Big\vert\hat{X}(S)\right)\\
&=E_{(i,j), Z\sim \mathcal{F}_{ij}}\left(Z^2+(M_{ij}-\hat{X}(S)_{ij})^2\Big\vert\hat{X}(S)\right)
=E_{(i,j), Z\sim \mathcal{F}_{ij}}\left(Z^2\right)+\frac{1}{nm}|M-\hat{X}(S)|^2_2\\
&=\sigma^2+\frac{1}{nm}|M-\hat{X}(S)|^2_2\;.\end{align*}

 Therefore, following the same reasoning as the proof of Theorem~\ref{L2bound} (and Remark~\ref{rem:noisebound}, we have that if $s\geq \mathbf{O}\left(\frac{r(n+m)}{\e}\cdot\frac{\sigma^2+\e}{\e}\cdot\log^3(r/\e)\right)$, then with high probability,
$$L(\hat{X})\leq \sigma^2+ \epsilon\;.$$
Applying the work above, we obtain
\begin{equation}\label{eq:ind_error_bound} \frac{1}{nm}|M-\hat{X}(S)|^2_2\leq \epsilon\;.\end{equation}

Now we turn to sampling without replacement. We first state a lemma which is proved in the appendix. (Notation: here $\D^s$ and $\D^s_{w/o}$ again denote sampling with or without replacement, but in this context $\D^s$ represents sampling with replacement when the noise is added independently each time an entry is sampled, as in the statement of Theorem~\ref{L2bound_inderror}.)

\begin{lemma}\label{lem:w_or_wo:ind}  Let $\X$ denote any class of matrices, with $\D^s$ and $\D^s_{w/o}$ defined as above. For any $c$, if $s$ satisfies $s\leq \tfrac{K+1}{e} (nm)^{1-\tfrac{1}{K+1}}$, then
$$P_{S\sim \mathcal{D}^s_{w/o}}\left(\sup_{X\in\X}g(L(X))-\hat{L}_S(X)\geq c\right)\leq 4K\cdot P_{S\sim \D^s}\left(\sup_{X\in\X}g(L(X))-\hat{L}_S(X)\geq (2K)^{-1}c\right)\;.$$
\end{lemma}

As in the proof of the sampling-without-replacement case of Theorem~\ref{L2bound}, this is sufficient to show that $\frac{1}{nm}|M-\hat{X}(S)|^2_2\leq 4K\cdot \epsilon$ with high probability for the stated sample complexity, as long as we also have that $s\leq \tfrac{K+1}{e} (nm)^{1-\tfrac{1}{K+1}}$.

\section{Comparison to prior work}

Suppose $Y=M+Z$ where $\rank(M)\leq r$ and $Z$ is a ``noise'' matrix of
average squared magnitude $\sigma^2=\frac{1}{nm}\norm{Z}^2_2$, and we
observe random entries of $Y$.  One might then consider different
types of reconstruction guarantees, requiring different assumptions on
$M$, $Z$ and the sampling distribution:
$$
\begin{array}{ll}
  \text{Exact recovery of }M:&\hat{X}(S)=M\;.\\
  \text{Near-exact recovery of }M:&\frac{1}{nm}|\hat{X}(S)-M|^2_2\leq \e
  \cdot \sigma^2\;.\\
  \text{Approximate recovery of }M:&\frac{1}{nm}|\hat{X}(S)-M|^2_2\leq \e \cdot \text{scale}(M)\;.\\
  \text{Approximate recovery of }Y:&\frac{1}{nm}|\hat{X}(S)-Y|^2_2\leq \sigma^2 +
  \epsilon\cdot\text{scale}(M)\;.
\end{array}
$$

Exact or near-exact recovery require strong incoherence-type
assumptions on the matrix $M$, and is not possible for arbitrary
low-rank matrices (see, e.g.~\citet{CandesRecht}).  Here we do not make any such
assumptions, and show that approximate recovery is still possible.
Such approximate recovery must be relative to some measure of the
scale of $M$, and we discuss results relative to both the maximal
magnitude, $\text{scale}(M)=\infnorm{M}^2$, and the average squared magnitude
$\text{scale}(M)=\frac{1}{nm}\norm{M}_2^2$.  Although not actually guaranteeing the
same type of ``recovery'', in Section \ref{sec:ExactRec} we nevertheless compare
the sample complexity required for our approximate recovery results to
the best sample complexity guarantee for exact and near-exact
recovery (obtained by \citet{Recht} and \citet{KMO}, respectively), and comment on the differences between the required
assumptions on $M$.

More directly comparable to our results are recent results by \citet{KMO}, \citet{NW}
and \citet{K} on approximate recovery of $M$.  These give essentially
the same type of guarantee as in Theorem~\ref{L2bound_inderror}, and also rely on
$\infnorm{M}^2$ as a measure of scale.  In Section \ref{sec:RelativeRec} we compare our
guarantee to these results, discussing the different dependence on the
various parameters and different assumptions on the noise. (Note that both types of results appear in \citet{KMO}; 
in Section~\ref{sec:RelativeRec}, we refer to the approximate recovery result stated in Theorem 1.1 of their paper,
while in Section~\ref{sec:ExactRec}, we refer to the near-exact recovery result stated in Theorem 1.2 of their paper.)

Recovery of $M$, whether exact, near-exact, or approximate, also
requires the noise to be independent and zero-mean, otherwise $M$
might not be identifiable.  All prior matrix reconstruction results we
are aware of work in this setting.  Approximate recovery of $M$ also
immediately implies an excess error bound on approximate recovery of
$Y$.  However, we also provide excess error bounds for approximate
recovery of $Y$, that do {\em not} assume independent nor zero-mean
noise (Theorems \ref{L1bound} and \ref{L2bound}).  That is, we provide
reconstruction guarantees in a significantly less restrictive setting
compared to other matrix reconstruction guarantees.

Another difference between different results is whether entries are
sampled with or without replacement, and if replacement is allowed,
whether the error is per-entry (i.e.~repeat observations of the same
entry are identical) or per-observation (i.e.~repeat observations of
the same entry are each corrupted independently).  However, as we show
in Sections \ref{proof:withoutrepl} and \ref{proof:ind_error}, and as has also been shown for exact
recovery \citep{Recht}, these differences do not significantly alter
the quality of reconstruction or the required sample size.

The most common algorithm for low-rank matrix recovery in the
literature is squared-error minimization subject to a penalty on trace
norm. All the methods cited here prove results about some variation of
this approach, with the exception of a recent result by \citet{KMO},
which applies to the output of the local search procedure
\textsc{OptSpace}.  In contrast, our results are mostly for error
minimization subject to a {\em max-norm} constraint.

\subsection{Comparison With Recent Approximate Recovery Guarantees}\label{sec:RelativeRec}

\citet{NW} and \citet{K} recently presented guarantees on approximate
recovery using trace-norm regularization, in a setting very similar to
our Theorem~\ref{L2bound_inderror}. Earlier work by \citet{KMO} uses a low-rank SVD approximation to $\tilde{Y}_S$
in the same setting to also obtain an approximate recovery guarantee. (Here $Y_S$ is the matrix
consisting of all observed entries of $Y$, with zeros elsewhere, and $\tilde{Y}_S$ is the same matrix
 with overrepresented rows and columns removed.)  In particular, each of the three guarantees provide
an $\epsilon$-approximate reconstruction of $M$ relative to
$\infnorm{M}^2$. That is, when $\infnorm{M}\leq 1$ as in Theorem~\ref{L2bound_inderror},
they provide the exact same guarantee $\frac{1}{nm}\norm{\hat{X}(S)-M}
\leq \epsilon$. (\citeauthor{NW} state the result relative to
$\frac{1}{nm}\norm{M}^2_2$, but have a linear dependence on the
``spikiness'' $\frac{\infnorm{M}}{\norm{M}_2/\sqrt{nm}}$,
effectively giving a guarantee relative to $\infnorm{M}^2$).

Specifically, assuming $|M|_{\infty}=1$ without loss of generality,
 \citeauthor{NW} and \citeauthor{K} assume the noise is
independent and subgaussian (or subexponential) with variance $\mathbf{O}(\sigma^2)$,
and require a sample size of:
\begin{equation}\label{SC_NW_Kolt}
s\geq \mathbf{O}\left(\frac{rn\log(n)}{\e}\cdot (1+\sigma^2)\right)\;.
\end{equation}
where the sample is drawn with replacement---in particular, an entry $(i,j)$ of the matrix
which is sampled multiple times gives multiple independent estimates of $M_{ij}$.

\citeauthor{KMO} give a result on approximate recovery which holds with no assumption on the noise, but requires additional assumptions such as i.i.d. noise to be a meaningful bound. The estimator used is the rank-$r$ SVD approximation to $\tilde{Y}_S$, defined above. Specifically, they show that, for sufficiently large sample size, with high probability, $\frac{1}{\sqrt{nm}}|\hat{X}(S)-M|_2\leq \mathbf{O}\left(\frac{nr\sqrt{n/m}}{s}+\frac{nmr}{s^2}\|\tilde{Z}_S\|_2^2\right)$, where $\tilde{Z}_S$ is defined in the same way as $\tilde{Y}_S$. For this bound to be meaningful, there must be some distributional assumption on $Z$---otherwise, we could have $\|Z_S\|_2\approx |Z_S|_2=\mathbf{O}(\sqrt{s})$, and the bound on mean error would actually increase with $\frac{nm}{s}$, and is thus not a meaningful bound. In the presence of i.i.d. subgaussian noise, however, \citeauthor{KMO} show that with high probability, $\|\tilde{Z}_S\|^2_2\leq\frac{ \sigma^2(\sqrt{n/m})s\log(s)}{m}$. Using this, approximate recovery of $M$ is obtained for sample complexity
\begin{equation}\label{SC_KMO_approxrec}
s\geq \mathbf{O}\left(\frac{rn}{\e}\cdot (\sqrt{n/m})\cdot\left(1+\log(n)\sigma^2\right)\right)\;,
\end{equation}
where the sample is drawn {\em without} replacement.  Therefore we may
regard \citeauthor{KMO}'s result as bounding error under the
assumption of i.i.d. subgaussian noise (or perhaps some weaker assumption
that gives the same result, such as independent subgaussian noise that
might not be i.i.d., or similar).  The guarantees
\eqref{SC_KMO_approxrec} and \eqref{SC_NW_Kolt} are therefore quite
similar, even though they are for fairly different methods, with
\eqref{SC_KMO_approxrec} being better when $\sigma^2=\mathbf{o}(1)$ but worse
for highly rectangular matrices.

\ignore{
Comparing the sample complexities~(\ref{SC_NW_Kolt}) and~(\ref{SC_KMO_approxrec}),
we see that the $\log(n)$ term in~(\ref{SC_KMO_approxrec}) is scaled by $\sigma^2$, which is
an important improvement over~(\ref{SC_NW_Kolt}) if the noise is much smaller than $\mathbf{O}(1)$ 
(or, with general scaling, much smaller than $\mathbf{O}(|M|^2_{\infty})$). On the other hand,~(\ref{SC_KMO_approxrec})
includes a factor of $\sqrt{n/m}$, which may be very detrimental for some data sets where $m$ is on
a much smaller scale than $n$, e.g. $m=\mathbf{O}(\sqrt{n})$.}

Comparing our Theorems~\ref{L2bound} and~\ref{L2bound_inderror} to the
above, the advantages of our results are:
\begin{itemize}
\setlength{\itemsep}{1pt}
  \setlength{\parskip}{0pt}
  \setlength{\parsep}{0pt}
\item We avoid the extra logarithmic dependence on $n$.
\item Even in order to guarantee recovery of $M$, we assume only a much
  milder condition on the noise: that noise is mean-zero, and that 
  with high probability, $|Z_S|_{\infty}\leq\sqrt{\frac{rn}{\log n}}$.  We
  do not assume the noise is identically distributed, nor subgaussian
  or subexponential.
\item We provide a guarantee on the excess error of recovering $Y$,
  even when the noise is {\em not} zero-mean nor independent.
\end{itemize}
The deficiency of our result is a possible slower rate of error
decrease: when $\sigma>0$ and $\epsilon=\mathbf{o}(\sigma^2)$ (i.e.~to
get ``estimation error'' significantly lower then the ``approximation
error''), our sample complexity scales as
$\tilde{\mathbf{O}}(1/\epsilon^2)$ compared to just
$\mathbf{O}(1/\epsilon)$ in the other results.  We do not know if this
difference represents a real consequence of not assuming zero-mean
independent noise in our analysis, or just looseness in the proof.
Our results also include an additional $\log^3(1/\epsilon)$ factor,
which we believe is purely an artifact of the proof technique.

\iffalse
We point out that, except for the quadratic dependence on $\epsilon$,
the results of \citeauthor{NW} and \citeauthor{K} matches our
``unsatisfying'' guarantee \eqref{eq:unsatisfying} for squared-error
reconstruction with the trace-norm.  However, as discussed in Section
\ref{proof:tracenorm}, we conjecture that using trace-norm, it should
be possible to relax the $\infnorm{M}^2$-scaling to an on-average
scaling, as is the case for absolute-error reconstruction (Theorem
\ref{L1bound}).
\fi

A strength of our analysis, as compared to that of \citeauthor{NW} and \citeauthor{K},
 is that the cases of sampling with and without replacement are both covered, including the case
 of per-entry noise when sampling with replacement, while
the results of \citeauthor{NW} and \citeauthor{K} are for sampling
with replacement with per-observation noise. This is an important improvement because in many applications, the observed
entries are drawn from a fixed matrix which was randomly generated, meaning that it is not possible
to obtain multiple independent observations of any $M_{ij}$.

\subsection{Comparison of results on exact and near-exact
  recovery}\label{sec:ExactRec}
  
The results of \citeauthor{Recht} and of \citeauthor{KMO} show that exact or near-exact recovery of the underlying low-rank matrix $M$
 can be obtained with high probability, when strong conditions on $M$ are assumed, and when the observations
 are either noiseless (for \citeauthor{Recht}'s exact recovery result) or are corrupted by i.i.d. subgaussian noise (for \citeauthor{KMO}'s
 near-exact recovery result).
 
 These results cannot be directly compared to the results we obtain
 in this paper, because the guarantees on recovery given by this work and by our work are fundamentally
 different---for instance, the error bound $\epsilon$ has completely different meanings
 in our definitions of near-exact recovery and approximate recovery above. These two incomparable
 types of guarantees are linked to very different conditions on the data---exact and near-exact recovery 
 cannot be obtained without strict assumptions about how the observations are generated.

 Nonetheless, one comparison between
 these methods which can be made, is in the magnitude of the required sample complexities to obtain some
 meaningful bound via each result---exact recovery for \citeauthor{Recht}'s result, near-exact recovery for \citeauthor{KMO}'s result,
 and approximate recovery relative to $\infnorm{M}^2$ for our result. The rest of this section is organized as follows:
we summarize the results in the literature  in Section~\ref{sec:ExactRec:RandKMO}, compare sample complexities in
Section~\ref{sec:ExactRec:SC}, and describe how incoherence is sufficient but not necessary for approximate
recovery relative to  $\frac{1}{nm}\norm{M}_2^2$  (instead of $\infnorm{M}^2$) 
in Sections~\ref{sec:ExactRec:MaxMu} and~\ref{sec:ExactRec:MaxNotMu}.\ignore{ We first summarize
 the results obtained by \citeauthor{Recht} and by \citeauthor{KMO} for exact and near-exact recovery, respectively, in Section~\ref{sec:ExactRec:RandKMO}. We compare
 the necessary sample complexities for these two results and for our result, in Section~\ref{sec:ExactRec:SC}.
 It is also interesting to note that incoherence type conditions
  similar to those of \citeauthor{KMO} are sufficient for getting approximate
  reconstruction relative to $\frac{1}{nm}\norm{M}_2^2$ using the
  max-norm (instead of relative to $\infnorm{M}^2$), which we discuss in Section~\ref{sec:ExactRec:MaxMu}.
  However, these incoherence conditions are clearly not necessary for approximate recovery
  with the max-norm, and in Section~\ref{sec:ExactRec:MaxNotMu} we give examples of cases when
  approximate recovery via the max-norm is possible but exact and near-exact recovery is impossible.}

\subsubsection{Details on exact and near-exact results in the literature}\label{sec:ExactRec:RandKMO}
Let $M=U\Sigma V^T$ be a reduced SVD of $M$. Let $\kappa$ be the condition number of $\Sigma$. Define also the incoherence parameters for matrix $M$ \citep{CandesRecht}:
$$\mu_0=\max\left\{\frac{n}{r}\cdot \max_i |U_{(i)}|^2_2,\frac{m}{r}\cdot \max_j |V_{(j)}|^2_2\right\}\;,$$
$$\mu_1=\sqrt{\frac{nm}{r}}\cdot \max_{i,j}|U_{(i)}^TV_{(j)}|\;,$$
where $U_{(i)}$ denotes the $i$th row of $U$ and $V_{(j)}$ denotes the $j$th row of $V$.

Suppose that $M$ has low incoherence parameters and $Z=0$. Improving on the earlier results of \citet{CandesRecht} and \citet{CandesTao}, \citeauthor{Recht} proves that $\hat{X}(S)=M$ (that is, exact recovery is obtained) with high probability if
\begin{equation}\label{SC_Recht}s\geq \mathbf{O}\left(rn\max\{\mu_0,\mu_1^2\}\log^2n\right)\;.\end{equation}

 In the case of noisy observations, \citeauthor{KMO} give conditions on low $\ell_2$ error in recovery (with high probability) in the setting of i.i.d. subgaussian noise with incoherent $M$, improving on \citet{CandesPlan} earlier work on the noisy case.
 (More precisely, \citeauthor{KMO} give a result which holds with no assumption on the noise, but requires additional assumptions such as i.i.d. noise to be a meaningful bound. We therefore regard their result as assuming i.i.d. subgaussian noise---see the discussion of their approximate reconstruction result above in Section~\ref{sec:RelativeRec}.)
  Their \textsc{OptSpace} algorithm is a method for finding the rank-$r$ matrix $\hat{X}$ minimizing squared error on the observed entries. Let $\hat{X}(S)$ denote the matrix recovered by this algorithm.  When the entries of $Z$ are i.i.d. subgaussian, \citeauthor{KMO}  show that, with high probability, if $s$ satisfies
\begin{equation}\label{SC_KMO}s\geq\mathbf{O}\left(rn\kappa^4\cdot\max\left\{\frac{1}{\e}\log\left(\frac{rn\kappa^4}{\e}\right),r\kappa^2\mu_0^2,r\kappa^2\mu_1^2\right\}\right)\;,\end{equation}
then $|\hat{X}(S)-M|^2_2\leq|Z|^2_2\cdot \epsilon$.
(For simplicity of the comparison, we use a slightly relaxed form of their required sample complexity, and ignore $\sqrt{n/m}$ in their error and sample bounds.)

\subsubsection{Comparing sample complexities}\label{sec:ExactRec:SC}

Ignoring the dependence on $\epsilon$, which as we discussed earlier
is in any case incomparable between approximate and exact and
near-exact recovery, our sample complexity for approximate recovery
using the max-norm is $\mathbf{O}(rn)$.  Even with ``perfect'' incoherence
parameters, this a factor of $\log^2(n)$ less then the sample
complexity established by \citeauthor{Recht} for exact recovery
\eqref{SC_Recht}, and a factor of $r$ less then the sample complexity
established by \citeauthor{KMO} for near-exact recovery~\eqref{SC_KMO}.  Of course,
``bad'' incoherence parameters may sharply increase the sample
complexity for exact or near-exact recovery, but do not affect our
sample complexity for approximate recovery.

\ignore{
For simplicity, we assume that $s\leq\frac{K+1}{e}(nm)^{1-\tfrac{1}{K+1}}$ holds for some finite $K$; for example we might have $m\approx n$ in which case this bound is reasonable in light of all the sample complexities. (Note that we also ignore the term $\sqrt{n/m}$ in \citeauthor{KMO}'s sample complexity.) The result of Theorem~\ref{L2bound_inderror} (with appropriate rescaling, over the class $\{X:\|X\|_{\max}\leq \sqrt{r}|M|_{\infty},|X|_{\infty}\leq|M|_{\infty}\}$) shows that, with high probability,
$$s\geq\mathbf{O}\left(rn\cdot \frac{\log^3(1/\e)}{\e}\cdot\frac{\sigma^2+\e}{\e}\right) \ \Rightarrow \ \frac{1}{nm}|M-\hat{X}(S)|^2_2\leq |M|^2_{\infty}\cdot\e\;.$$

In contrast, the sample complexity required for exact recovery by
\citeauthor{Recht} is $\geq\mathbf{O}(rn\log^2n)$, and for near-exact
recovery by \citeauthor{KMO} is
$\geq\mathbf{O}\left(\frac{rn}{\e}\log\left(\frac{rn}{\e}\right)\right)$,
with each sample complexity depending also on the incoherence
parameters and/or the condition number of $\Sigma$, which might be
quite large (e.g. $\mu_0$ may attain an upper bound of $\frac{n}{r}$).
}

\subsubsection{Approximate recovery relative to average signal
  magnitude, in the presence of incoherence
  conditions}\label{sec:ExactRec:MaxMu}

It is interesting to note that the incoherence assumptions, used by \citeauthor{Recht} and by
\citeauthor{KMO}, enable approximate recovery with the max-norm
relative to the average magnitude $\frac{1}{nm}\norm{M}^2_2$, and not
only the maximal magnitude, as in Theorem~\ref{L2bound}.  This is
based on the following observation:

\ignore{We first relate the result implied by Theorem~\ref{L2bound} to the incoherence parameters, to facilitate comparison. Suppose the entries of $Z$ are independent, with variance $\sigma^2$ (where perhaps $\sigma^2=0$), and assume $|Z|_{\infty}\leq |M|_{\infty}\sqrt{\frac{n}{\log n}}$ (true with extremely high probability with an i.i.d. subgaussian assumption, but also true with high probability with some milder moments assumptions on $Z$).}

\begin{lemma}\label{lem:max_vs_kappamu} Let $M\in\R^{n\times m}$ and let $\kappa$ and $\mu_0$ be defined as before. Then
$$\|M\|_{\max}\leq \min\{\kappa,\sqrt{r}\}\mu_0\sqrt{r}\cdot\frac{|M|_2}{\sqrt{nm}}\;.$$
In particular, by Lemma~\ref{lem:max_vs_infinity}, the above expression is also an upper bound for $|M|_{\infty}$. \end{lemma}
\begin{proof}
First, observe that
$$\|M\|_{\max}\leq\max_{i,j}|(U\Sigma)_{(i)}|_2\cdot|V_{(j)}|_2\leq \sigma_1\cdot\max_{i,j}|U_{(i)}|_2\cdot|V_{(j)}|_2\leq \sigma_1\cdot\frac{\mu_0r}{\sqrt{nm}}\;.$$
Also,
\begin{equation*}
\sigma_1\leq \kappa\sqrt{\sigma_r^2}\leq \frac{\kappa}{\sqrt{r}}\sqrt{\sigma_1^2+\dots+\sigma_r^2}=\frac{\kappa|M|_2}{\sqrt{r}}\text{ and }
\sigma_1\leq \sqrt{\sigma_1^2+\dots+\sigma_r^2}=|M|_2\;.\qedhere
\end{equation*}
\end{proof}
Now, based on Remark \ref{rem:L2result}, if
$\frac{1}{nm}\norm{M}_2^2\leq 1$ (and with a mild bound on $\infnorm{Z}$),
with high probability over a sample of size
\begin{equation}\label{SC_ours}s\geq \mathbf{O}\left(\frac{rn}{\e}\cdot\frac{\sigma^2+\e}{\e}\cdot
      \min\{\kappa^2,r\}\mu_0^2\cdot\log^3\left(\frac{\mu_0^2r}{\e}\right)\right) \;,\end{equation}
we have $|Y-\hat{X}(S)|^2_2\leq \sigma^2 + \e$.  Up to log factors and the
dependence on $\epsilon$, this sample complexity is at most as much as the
sample complexity required by \citeauthor{KMO}, given in~\eqref{SC_KMO}.

\subsubsection{Approximate recovery relative to average signal magnitude, in the absence of incoherence conditions}\label{sec:ExactRec:MaxNotMu}

We make note of several special cases where using max-norm and the concentration result, and bounding excess error relative to $\tfrac{1}{nm}|M|^2_2$, may compare more favorably to other methods than the results above would indicate.

\begin{itemize}
\item If $U=V$ (that is, $M$ is symmetric), then $\mu_1=\mu_0\sqrt{r}$ and so our sample complexity compares more favorably to the sample complexities obtained by \citeauthor{Recht} and \citeauthor{KMO} (which both involve $\mu_1^2$).
\item Our sample complexity uses Lemma~\ref{lem:max_vs_kappamu} to bound $\|M\|_{\max}$ relative to $\tfrac{1}{\sqrt{nm}}|M|_2$. An example where $\kappa=1$ and $\|M\|_{\max}\ll \frac{\mu_0\sqrt{r}|M|_2}{\sqrt{nm}}$ (i.e. the bound in Lemma~\ref{lem:max_vs_kappamu} is extremely loose) is the case where the spiky columns of $U$ do not align with the spiky columns of $V$, for example writing $n=m=N+1$ we have:
$$M=\left(\begin{array}{cc}1&0\\0&N^{-1/2}\\0&N^{-1/2}\\\dots&\dots\\0&N^{-1/2}\\\end{array}\right)\cdot \left(\begin{array}{cc}0&1\\N^{-1/2}&0\\N^{-1/2}&0\\\dots&\dots\\N^{-1/2}&0\\\end{array}\right)^T = \left(\begin{array}{cc}N^{-1/4}&0\\0&N^{-1/4}\\0&N^{-1/4}\\\dots&\dots\\0&N^{-1/4}\\\end{array}\right)\cdot\left(\begin{array}{cc}0&N^{-1/4}\\N^{-1/4}&0\\N^{-1/4}&0\\\dots&\dots\\N^{-1/4}&0\\\end{array}\right)^T \;.$$
Since the left-hand factorization is an SVD of $M$ (omitting $\Sigma=I_2$), we therefore have $\mu_0\sqrt{r}\cdot\tfrac{|M|_2}{\sqrt{nm}}=1$ while the right-hand factorization shows that  $\|M\|_{\max}\leq \tfrac{1}{\sqrt{n-1}}$.
\item Large condition numbers $\kappa$ can often lead to the same situation, in which the max norm is far lower than the bound implied by Lemma~\ref{lem:max_vs_kappamu}. For example, if low-rank $M$ is a matrix where $\|M\|_{\max}\approx \frac{\kappa\mu_0\sqrt{r}\cdot|M|_2}{\sqrt{nm}}$, but if we perturb $M$ slightly and add an extremely low singular value, then $\kappa$ becomes extremely high while $\|M\|_{\max}$ is only slightly perturbed.
\end{itemize}

\section{Summary}

We presented low rank matrix reconstruction guarantees based on an
existing analysis of the Rademacher complexity of low trace-norm and
low max-norm matrices, and carefully compared these to other recently
presented results.  We view the main contributions of this papers as:
\begin{itemize}
\setlength{\itemsep}{1pt}
  \setlength{\parskip}{0pt}
  \setlength{\parsep}{0pt}
\item Following a string of results on low-rank matrix reconstruction,
  showing that an existing Rademacher complexity analysis combined
  with simple arguments on the relationship between the rank, max-norm,
  and trace-norm, can yield guarantees that are in several ways
  better, and relying on weaker assumptions.
\item Pointing out that the max-norm can yield superior reconstruction
  guarantees over the more commonly used trace-norm.
\iffalse
\item Pointing out that, contrary to existing $\ell_2$-approximate
  reconstruction guarantees, the trace-norm {\em should} be able to
  yield approximate reconstruction relative to the {\em average},
  rather then maximal, entry magnitude.
\fi
\item Studying the issue of sampling with and without replacement, and
  establishing rigorous generic results relating the two settings.
  This has been done before for exact recovery \citep{Recht}, but is
  done here for the more delicate situation of approximate recovery of
  either $M$ or $Y$.
\end{itemize}
The main deficiency of our approach is a worse dependence on the
approximation parameter $\epsilon$, when $\sigma>0$ (i.e.~the
approximately low rank case) and $\epsilon=\mathbf{o}(\sigma^2)$
(i.e.~estimation error less then approximation error).  Although this
dependence is tight for general classes with bounded Rademacher
complexity, we do not know if it can be improved in Theorem
\ref{L2bound}.  In particular, we do not know whether the less
favorable dependence is a consequence of not relying on zero-mean
i.i.d.~noise, or not relying on $M$ having low-rank (instead of only assuming low max-norm),
 or on relying only on the
Rademacher complexity of the class of low max-norm matrices---perhaps
better bounds can be obtained with a more careful analysis.

\bibliography{mb}

\newpage

\appendix
\section{Proof of Sampling-Without-Replacement Lemmas}

\begin{proof}{\it(Lemmas~\ref{lem:w_or_wo:expectation} and~\ref{lem:w_or_wo:prob}).} 
Let $\mathbb{S}_r=\{S\in\mathcal{X}^s  :  \text{ each }x\in\mathcal{X}\text{ appears at most }r\text{ times in }S\}$. Let $S\sim\mathcal{D}^s_r$ denote a sample $S$ drawn uniformly from $\mathbb{S}_r$. In particular, $\mathcal{D}^s_0=\mathcal{D}^s_{w/o}$ and $\mathcal{D}^s_s=\mathcal{D}^s$. By Lemma~\ref{lem:w_or_wo:extra} (proved below), for any $r$,
\begin{align*}
E_{S\sim\mathcal{D}^s_{w/o}}\left(\sup_{h\in\H}L(h)-\hat{L}_{S}(h)\right)&\leq E_{S\sim\mathcal{D}^s_r}\left(\sup_{h\in\H}L(h)-\hat{L}_S(h)\right)\;,\\
P_{S\sim\mathcal{D}^s_{w/o}}\left(\sup_{h\in\H}g(L(h))-\hat{L}_{S}(h)\geq c\right)&\leq r!\cdot P_{S\sim\mathcal{D}^s_r}\left(\sup_{h\in\H}g(L(h))-\hat{L}_S(h)\geq c\right)\;.
\end{align*}
Taking the first inequality with $r=s$, this completes the proof for Lemma~\ref{lem:w_or_wo:expectation}.

Now we complete the proof of Lemma~\ref{lem:w_or_wo:prob}. Take $S\sim\mathcal{D}^s$ and write $S=\{e_1,\dots,e_s\}$. For any $i_1<i_2<\dots<i_{K+1}$,
$$P\left(e_{i_1}=e_{i_2}=\dots=e_{i_{K+1}}\right)=\frac{1}{(nm)^K}\;,$$
and so for any $K$ with $(K+1)!\geq 2s$, the probability that any entry of the matrix appears at least $(K+1)$ times in $S$ is bounded by
$${s\choose K+1}\cdot \frac{1}{(nm)^K}\leq \frac{s^{K+1}}{(K+1)!(nm)^K}\leq\frac{s}{(K+1)!}\leq \tfrac{1}{2}\;.$$
Fix the smallest $K$ such that $(K+1)!\geq 2s$. This implies $K!<2s$. We then have
\begin{align*}
&P_{S\sim\mathcal{D}^s_{w/o}}\left(\sup_{h\in\H}g(L(h))-\hat{L}_{S}(h)\geq c\right)\\&\leq K!\cdot P_{S\sim\mathcal{D}^s_K}\left(\sup_{h\in\H}g(L(h))-\hat{L}_S(h)\geq c\right)\\
&\leq K!\cdot \left(P_{S\sim\mathcal{D}^s}\left(\text{ each }x\in\mathcal{X}\text{ appears at most }K\text{ times in }S\right)\right)^{-1}\cdot P_{S\sim\mathcal{D}^s}\left(\sup_{h\in\H}g(L(h))-\hat{L}_S(h)\geq c\right)\\
&\leq 2K!\cdot P_{S\sim\mathcal{D}^s}\left(\sup_{h\in\H}g(L(h))-\hat{L}_S(h)\geq c\right)\\
&\leq 4s\cdot P_{S\sim\mathcal{D}^s}\left(\sup_{h\in\H}g(L(h))-\hat{L}_S(h)\geq c\right)\;.
\end{align*}
This is completes the proof for Lemma~\ref{lem:w_or_wo:prob}.

\end{proof}

\begin{lemma}\label{lem:w_or_wo:extra} Using the notation of the proof above, for any $r$,
\begin{align*}
E_{S\sim\mathcal{D}^s_{w/o}}\left(\sup_{h\in\H}L(h)-\hat{L}_{S}(h)\right)&\leq E_{S\sim\mathcal{D}^s_r}\left(\sup_{h\in\H}L(h)-\hat{L}_S(h)\right)\;,\\
P_{S\sim\mathcal{D}^s_{w/o}}\left(\sup_{h\in\H}g(L(h))-\hat{L}_{S}(h)\geq c\right)&\leq r!\cdot P_{S\sim\mathcal{D}^s_r}\left(\sup_{h\in\H}g(L(h))-\hat{L}_S(h)\geq c\right)\;.
\end{align*}
\end{lemma}
\begin{proof}

Write $\Omega=[n]\times[m]$. Let $\alpha(S)$ be any function of the sample $S$, where $S$ may contain repeated entries. Assume that, for any $S,S_1,\dots,S_r$ of equal size such that $r\cdot S=S_1+\dots+S_r$, $\alpha(\cdot)$ satisfies the following for some function $a(r)$:
\begin{equation}\label{eq:w_or_wo:alpha}a(r)\cdot \alpha(S)\leq \sum_{i=1}^r \alpha(S_i)\;.\end{equation}

Consider all samples from $\Omega$, drawn with replacement. For a sample set $S$ of size $s$, for $i=1,\dots,s$, let $N_i(S)$ equal the number of elements of $\Omega$ appearing exactly $i$ times in $S$, which obeys $\sum_i iN_i(S)=s$. We call $\N(S)=(N_1(S),\dots,N_s(S))$ the multiplicity vector of $S$; note that, when convenient, we might write $\N(S)$ to have length greater than $s$ (filling the last terms with zeros). From this point on, we will regard these samples as ordered lists, and assume that in any sample, $S$ is ordered in the format
$$(\o^1_1,\dots,\o^1_{N_1(S)},\o^2_1,\o^2_1,\dots,\o^2_{N_2(S)},\o^2_{N_2(S)},\o^3_1,\o^3_1,\o^3_1,\dots)\;,$$
where for any $i$ we might permute the $\o^i_j$'s.

Let $\N$ be any multiplicity vector, of the form $(N_1,\dots,N_r,0,\dots,0)$ for some $r\leq s$. Let $\N'$ and $\N''$ be multiplicity vectors derived from $\N$ as follows:
$$N'_i=\left\{\begin{array}{ll}N_1+rN_r,&i=1\\N_i,&2\leq i\leq r-1\\0,&i\geq r\\\end{array}\right., \ \ N''_i=\left\{\begin{array}{ll}N_i,&1\leq i\leq r-1\\0,&i\geq r\\\end{array}\right.$$
Define $s=\sum_i iN_i$. Note that $\sum_i iN'_i=s$ and $\sum_i iN''_i=s-rN_r$.

Let $\SS=\{S:\N(S)=\N\}$, $\SS'=\{S:\N(S)=\N'\}$, $\SS''=\{S:\N(S)=\N''\}$. We will first prove that $E_{S'\sim Unif(\SS')}\left[\alpha(S')\right]\leq E_{S\sim Unif(\SS)}\left[\alpha(S)\right]$, and then induct on $r$.

First consider $\SS'$. We have
$$|\SS'|E_{S'\sim Unif(\SS')}\left[\alpha(S')\right]=\sum_{S'\in\SS'}\left[\alpha(S')\right]=\sum_{S''\in\SS''}\sum_{\substack {A_1,\dots,A_r\subset \Omega\backslash S''\\|A_j|=N_r\\A_j\text{'s \ disjoint}}}\left[\alpha(S''+A_1+\dots+A_r)\right]\;.$$
The last equality arises when, starting with some $S'\in \SS'$, we recall that $S''$ is an ordered sample set beginning with the $N_1+rN_r$ elements which appear exactly once. Let $S''$ be the first $N_1$ elements of $S'$, then let $A_1$ be the next $N_r$ elements of $S'$, let $A_2$ be the next $N_r$ elements of $S'$, etc.

Next consider $\SS$. As before, we have
$$|\SS|E_{S\sim Unif(\SS)}\left[\alpha(S)\right]=\sum_{S\in\SS}\left[\alpha(S)\right]=\sum_{S''\in\SS''}\sum_{\substack{A\subset \Omega\backslash S\\|A|=N_r}}\left[\alpha(S''+r\cdot A)\right]\;.$$
By counting how many times each choice of $A$ appears in the sum below, and then rescaling accordingly, we get
$$=\left(\frac{(nm-N_1-\dots-N_r)!}{(nm-N_1-\dots-N_{r-1}-rN_r)!}\right)^{-1}r^{-1}\sum_{S''\in\SS''}\sum_{\substack{A_1,\dots,A_r\subset \Omega\backslash S''\\|A_j|=N_r\\A_j\text{'s \ disjoint}}}\sum_j \left[\alpha(S''+r\cdot A_j)\right]$$
$$\geq \left(\frac{(nm-N_1-\dots-N_r)!}{(nm-N_1-\dots-N_{r-1}-rN_r)!}\right)^{-1}\frac{a(r)}{r}\sum_{S''\in\SS''}\sum_{\substack{A_1,\dots,A_r\subset \Omega\backslash S''\\|A_j|=N_r\\A_j\text{'s \ disjoint}}}\alpha(S''+A_1+\dots+ A_r)\;.$$

To summarize so far, we have
$$|\SS|E_{S\sim Unif(\SS)}\left[\alpha(S)\right] \geq \left(\frac{(nm-N_1-\dots-N_r)!}{(nm-N_1-\dots-N_{r-1}-rN_r)!}\right)^{-1}\cdot\frac{a(r)}{r}|\SS'|E_{S'\sim Unif(\SS')}\left[\alpha(S')\right] \;.$$

Next, we see that (since sample sets are treated as ordered)
$$|\SS|=\frac{(nm)!}{(nm-N_1-\dots-N_r)!}, \ |\SS'|=\frac{(nm)!}{(nm-N_1-\dots-N_{r-1}-rN_r)!}$$

Therefore,
$$E_{S\sim Unif(\SS)}\left[\alpha(S)\right] \geq \frac{a(r)}{r}\cdot E_{S'\sim Unif(\SS')}\left[\alpha(S')\right] \;.$$

By inducting over $r$, we then see that
$$E_{S\sim Unif(\SS)}\left[\alpha(S)\right] \geq \frac{\prod_{i=1}^r a(i)}{r!}\cdot E_{S\sim\mathcal{D}^s_{w/o}}\left[\alpha(S)\right] \;,$$
where $\SS=\{S:\mathbf{N}(S)=\mathbf{N}\}$ for any multiplicity vector $\mathbf{N}=(N_1,\dots,N_r,0,\dots,0)$. Therefore,
$$E_{S\sim \mathcal{D}^s_r}\left[\alpha(S)\right] \geq \frac{\prod_{i=1}^r a(i)}{r!}\cdot E_{S\sim\mathcal{D}^s_{w/o}}\left[\alpha(S)\right] \;,$$

Finally, we observe that if $\alpha(S)=\sup_{h\in\H}L(h)-\hat{L}_S(h)$, then $\alpha(S)$ satisfies~(\ref{eq:w_or_wo:alpha}) with $a(r)=r$, while if $\alpha(S)=\mathbb{I}\left\{\sup_{h\in\H}g(L(h))-\hat{L}_S(h)\geq c\right\}$, then $\alpha(S)$ satisfies~(\ref{eq:w_or_wo:alpha}) with $a(r)=1$. This concludes the proof.

\end{proof}

\begin{proof}{\it (Lemma~\ref{lem:w_or_wo:ind}.)}

Suppose  $s\leq \tfrac{K+1}{e} (nm)^{1-\tfrac{1}{K+1}}$. Then, as in the proof of Lemma~\ref{lem:w_or_wo:prob}, 
\begin{align*}P\left(\text{any entry is sampled more than }K\text{ times}\right)&\leq {s\choose K}\cdot\frac{1}{(nm)^{K-1}}\\
&\leq \frac{s^{K+1}}{(K+1)!(nm)^{K}}\\
&\leq \frac{(K+1)/e)^{K+1}(nm)^{K}}{(K+1)!(nm)^{K}}\\
&\leq\frac{1}{2},\text{ by Stirling's approximation.}\end{align*}

We show below that, for any $c$,
$$P_{S\sim \mathcal{D}^s_{w/o}}\left(\sup_{X\in\X}g(L(X))-\hat{L}_S(X)\geq c\right)\leq 2K\cdot P_{S\sim \mathcal{D}^s_K}\left(\sup_{X\in\X}g(L(X))-\hat{L}_S(X)\geq (2K)^{-1}c\right)\;,$$
where $\mathcal{D}^s_{w/o}$ and $\mathcal{D}^s_K$ are defined as in the proof of Lemmas~\ref{lem:w_or_wo:expectation} and~\ref{lem:w_or_wo:prob}, except with the independent noise model. As in the proof of Lemmas~\ref{lem:w_or_wo:expectation} and~\ref{lem:w_or_wo:prob}, this implies that
$$P_{S\sim \mathcal{D}^s_{w/o}}\left(\sup_{X\in\X}g(L(X))-\hat{L}_S(X)\geq c\right)\leq 4K\cdot P_{S\sim \D^s}\left(\sup_{X\in\X}g(L(X))-\hat{L}_S(X)\geq (2K)^{-1}c\right)\;.$$

We now prove that, for any $c$,
$$P_{S\sim \mathcal{D}^s_{w/o}}\left(\sup_{X\in\X}g(L(X))-\hat{L}_S(X)\geq c\right)\leq 2K\cdot P_{S\sim \mathcal{D}^s_K}\left(\sup_{X\in\X}g(L(X))-\hat{L}_S(X)\geq (2K)^{-1}c\right)\;.$$

Write $\Omega=[n]\times[m]$. Consider all samples from $\Omega$, drawn with replacement. When a particular $(i,j)$ is drawn multiple times, then the observed values at that entry of the matrix follow the independent noise model as described in the statement of Theorem~\ref{L2bound_inderror}.

For a sample set $S$ of size $s$, for $i=1,\dots,s$, define $\N(S)$ as in the proof of Lemma~\ref{lem:w_or_wo:extra}. Let $\N$ be any multiplicity vector, of the form $(N_1,\dots,N_r,0,\dots,0)$ for some $r\leq s$. Let $\mathbf{M}$ be a multiplicity vector defined from $\N$ as follows:
$$\mathbf{M}=(M_i)_i,\text{ where }M_i=N_{2i-1}+2N_i+N_{i+1}\;.$$

Now take any $A_1,A_2,\dots,A_{2r},B_2,\dots,B_{2r}\subset[n]\times[m]$, all disjoint, with $|A_i|=|B_i|=N_i$ for all $i$. Define $B_1=A_1$, and
$$S_A=\sum_{i=1}^{2r}\left(\sum_{j=1}^i A_i^{(j)}\right)\;, \ S_B=\sum_{i=1}^{2r}\left(\sum_{j=1}^i B_i^{(j)}\right)\;.$$
Note that $\N(S_A)=\N(S_B)=\N$. Now define
$$T_1=\sum_{i=1}^{2r}\left(\sum_{j=1}^{\lfloor \tfrac{i}{2}\rfloor}A_i^{(j)}+\sum_{j=\lfloor\tfrac{i}{2}\rfloor+1}^i B_i^{(j)}\right)\;, \  T_2=\sum_{i=1}^{2r}\left(\sum_{j=1}^{\lfloor \tfrac{i}{2}\rfloor}B_i^{(j)}+\sum_{j=\lfloor\tfrac{i}{2}\rfloor+1}^i A_i^{(j)}\right)\;.$$
Note that $\N(T_1)=\N(T_2)=\mathbf{M}$, and that up to reordering, $S_A+S_B=T_1+T_2$. We treat $T_1$ and $T_2$ as functions of $(S_A,S_B)$.

Write $\alpha_c(S)=\mathbb{I}\left\{\sup_{X\in\X}g(L(X))-\hat{L}_S(X)\geq c\right\}$. Then $\alpha$ satisfies the following whenever $|S_1|=|S_2|$:
$$\frac{1}{2}\left(\alpha_{2c}(S_1)+\alpha_{2c}(S_2)\right)\leq\alpha_c(S_1+S_2)\leq\alpha_c(S_1)+\alpha_c(S_2)\;.$$
Therefore,
\begin{align*}
2E_{S\sim Unif(\N)}\left(\alpha_c(S)\right)
&=(\#(S_A,S_B)\text{ pairs as above})^{-1}\sum_{(S_A,S_B)\text{ as above}}\alpha_c(S_A)+\alpha_c(S_B)\\
&\geq (\#(S_A,S_B)\text{ pairs as above})^{-1}\sum_{(S_A,S_B)\text{ as above}}\alpha_c(S_A+S_B)\\
&=(\#(S_A,S_B)\text{ pairs as above})^{-1}\sum_{(S_A,S_B)\text{ as above}}\alpha_c(T_1+T_2)\\
&\geq (\#(S_A,S_B)\text{ pairs as above})^{-1}\sum_{(S_A,S_B)\text{ as above}}\frac{1}{2}\left(\alpha_{2c}(T_1)+\alpha_{2c}(T_2)\right)\\
&= (\#(S_A,S_B)\text{ pairs as above})^{-1}\sum_{(S_A,S_B)\text{ as above}}\alpha_{2c}(T_1)\\
&= (\#(S_A,S_B)\text{ pairs as above})^{-1}\sum_{T:\N(T)=\mathbf{M}}\alpha_{2c}(T)\cdot(\#(S_A,S_B)\text{ pairs such that }T=T_1)\\
\end{align*}

We also have the following (note that here we treat samples as unordered, unlike in the proofs of Lemmas~\ref{lem:w_or_wo:expectation} and~\ref{lem:w_or_wo:prob}):
$$(\#(S_A,S_B)\text{ pairs as above})={nm\choose N_1,N_2,N_2,N_3,N_3,\dots,N_{2r},N_{2r}}\;,$$
and for any $T$ with $\N(T)=\mathbf{M}$,
$$(\#(S_A,S_B)\text{ pairs such that }T=T_1)=\prod_{i=1}^r{M_i\choose N_{2i-1},N_{2i},N_{2i},N_{2i+1}}\;.$$
Finally,
$$(\# T:\N(T)=\mathbf{M}(T))={nm\choose M_1,M_2,\dots,M_r}\;,$$
and therefore, continuing from above,
\begin{align*}
2E_{S\sim Unif(\N)}\left(\alpha_c(S)\right)
&=  (\#(S_A,S_B)\text{ pairs as above})^{-1}\sum_{T:\N(T)=\mathbf{M}}\alpha_{2c}(T)\cdot(\#(S_A,S_B)\text{ pairs such that }T=T_1)\\
&=(\#T:\N(T)=\mathbf{M})^{-1}\sum_{T:\N(T)=\mathbf{M}}\alpha_{2c}(T)\\
&=E_{T\sim Unif(\mathbf{M})}(\alpha_{2c}(T))\;.
\end{align*}

Inducting over $r$, we see that for any $\mathbf{N}=(N_1,\dots,N_r,0,\dots,0$,
$$2^{K(r)}E_{S\sim Unif(\N)}(\alpha_c(S))\geq E_{S\sim \mathcal{D}^s_{w/o}}(\alpha_{2^{K(r)}}(S))\;,$$
where $K(r)$ is the number of times that the operation $x\mapsto\lceil x/2\rceil$ must be applied iteratively to $r$ to obtain $1$; note that $2^{K(r)}\leq 2r$. Therefore,
$$P_{S\sim\mathcal{D}^s_{w/o}}\left(\sup_{X\in\mathcal{X}}g(L(x))-\hat{L}_S(X)\geq c\right)\leq 2r P_{S\sim\mathcal{D}^s_r}\left(\sup_{X\in\mathcal{X}}g(L(x))-\hat{L}_S(X)\geq (2r)^{-1}c\right)\;.$$

\end{proof}

\section{The Rademacher Complexity of the Trace-Norm Ball}\label{app:trnormrad}

\citet{ShraibmanSrebro} established that for a sample
$S=\{(i_1,j_1),\ldots,(i_s,j_s)\}$ of $s$ index-pairs, the empirical
Rademacher complexity of the trace-norm ball, viewed a predictor of
entries, is given by:
\begin{multline}
  \empRad_s\left(\left\{(i,j)\mapsto X_{ij} \,\middle|\, X \in
      \R^{n \times m}, \trnorm{X}\leq A \right\}\right) 
=  \Ep{\xi}{\sup_{\trnorm{X}\leq A} \frac{1}{s}\sum_{t=1}^s \xi_t
    X_{(i_t,j_t)}}
\\ = \frac{A}{s} \Ep{\xi}{\Norm{\sum_{t=1}^s \xi_t e_{i_t,j_t}}_2}\;,\label{eq:radcite}
\end{multline}
where the expectations is over independent uniformly distributed
random variables $\xi_1,\ldots,\xi_t \in \pm 1$, $\Norm{X}_2$ is the
spectral norm (maximal singular value) of $X$, and $e_{i,j}=e_i e_j^T$
is a matrix with a single $1$ at location $(i,j)$ and zeros elsewhere.
Analyzing the Rademacher complexity then amounts to analyzing the
expected spectral norm of the random matrix $Q=\sum_{t=1}^s \xi_t
e_{i_t,j_t}$.  

The worst-case Rademacher complexity, i.e.~the supermum of
\eqref{eq:radcite} over all samples $S$, is $\frac{1}{\sqrt{s}}$, and
does not lead to meaningful generalization results.  Indeed, if we
could meaningfully bound the worst-case Rademacher complexity, we
could guarantee learning under arbitrary sampling distributions over
index-pairs, but this is not the case---we know that trace-norm
regularization can fail when entries are not sampled uniformly
\citep{RusAndNatiNIPS}.

Instead, we focus on bounding the {\em expected} Rademacher
complexity, i.e.~the expectation of \eqref{eq:radcite} when entries in
$S$ are chosen independently from a {\em uniform} distribution over
index pairs.  \citet{ShraibmanSrebro} bounded the expected Rademacher
complexity by
$\mathbf{O}\left(\frac{A}{\sqrt{nm}}\sqrt{\frac{(n+m)\log^{3/2}n}{s}}\right)$ using a
bound of \citet{Seginer} on the spectral norm of a matrix with fixed
magnitudes and random signs, combined with arguments bounding the
number of observations in each row and column.  Here we present a much
simpler analysis, reducing the logarithmic factor from $\log^{3/2}(n)$
to $\log(n)$, using a recent result of \citet{TropTailBounds}.

We now proceed to bounding $\EE{\Norm{Q}_2}$, where the expectation is
over the sample $S$ and the random signs $\xi_t$. Denote $P_t=\xi_t
e_{i_t,j_t}$, we have $Q=\sum_t P_t$ and $P_t$ are i.i.d.~zero-mean
random matrices (recall that now both $\xi_t$ and $(i_t,j_t)$ are
random).  Theorem 6.1 of \citet{TropTailBounds}, combined with Remarks
6.3 and 6.5, allows us to bound the expected spectral norm of such a
sum of independent random matrices by:
\begin{equation}
  \label{eq:citeTropp}
  \EE{\Norm{Q}} = \mathbf{O}\left( \sigma \sqrt{\log(n+m)} + R \log(n+m) \right)\;,
\end{equation}
where $\Norm{P_t}_2\leq R$ (almost surely) and $$\sigma^2 =
\max\left(\,\Norm{\sum\EE{P_t^T P_t}}_2\, ,\,\Norm{\sum\EE{P_t P^T_t}}_2\,\right).$$  For
each $t$, $P_t$ is just a matrix with a single $+1$ or $-1$, hence
$\Norm{P_t}\leq 1$.  The matrix $P_t P^T_t \in \R^{n \times n}$ is
equal to $e_{i,i}$ with probability $\frac{1}{n}$, hence $\EE{P_t
  P_t^T}=\frac{1}{n}I_n$ and $\Norm{\sum\EE{P_t
    P^T_t}}_2=\Norm{\frac{s}{n}I_n}=\frac{s}{n}$.  Symmetrically,
$\Norm{\sum\EE{P^T_t P_t}}_2=\frac{s}{m}$ and so $\sigma^2
=\frac{s}{nm}\max(n,m)$.  Plugging $\sigma$ and $T$ into
\eqref{eq:citeTropp} we have:
\begin{equation}\label{eq:finalnormbound}
  \EE{\Norm{Q}_2} =  O\left( \sqrt{\frac{s(n+m)\log(n+m)}{nm}} + \log(n+m) \right)
  = O \left( \sqrt{\frac{s(n+m)\log(n+m)}{nm}} \right)
\end{equation}
where in the second inequality we assume $s\geq m$.  Plugging
\eqref{eq:finalnormbound} into \eqref{eq:radcite} we get:
\begin{equation}
  \label{eq:finalradbound}
  \EE{\empRad_s\left(\left\{(i,j)\rightarrow X_{ij} \,\middle|\, X \in
      \R^{n \times m}, \trnorm{X}\leq A \right\}\right)} 
= O \left( \frac{A}{\sqrt{nm}} \sqrt{\frac{(n+m)\log(n+m)}{s}} \right)
\end{equation}

\section{Using Rademacher complexity to bound error}\label{app:BM_radbound}

Let $\X$ be any class of matrices. We first discuss the $\ell_1$-loss case, in which we would like 
to bound $\frac{1}{nm}\left|Y-\hat{X}(S)\right|_1$, in expectation over the sample $S$ of size $s$ drawn
uniformly at random (with replacement) from the matrix. Regarding this as a prediction problem, this is equivalent to bounding $E_S\left(L(\hat{X}(S))\right)$.
We reformulate a result from~\citet{BartlettMendelson} in the following lemma.

\begin{lemma}\label{lem:BM_radbound} Let $M$ be any matrix in $\X$. Excess reconstruction error can be bounded in expectation as
$$E_S \left[L(\hat{X}(S))\right]\leq L(M)+2\RR_s(\X)\;,$$
where in this case $\RR_s(\X)$ denotes the expected Rademacher complexity over a sample of size $s$.
\end{lemma}
\begin{proof}
Combining Theorems 8 and 12(4) from~\citet{BartlettMendelson} (using the last part of the proof of Theorem 8 rather than the main statement), 
since $\ell_1$-error is a 1-Lipschitz function,
$$E_S \left[\sup_{X\in\X} (L(X)-\hat{L}(X))\right]\leq2 \RR_s(\X)\;.$$
In particular, this implies that
$$E_S \left[L(\hat{X}(S))-\hat{L}(\hat{X}(S))\right]\leq2 \RR_s(\X)\;.$$
Furthermore, for any sample $S$, by definition we know $\hat{L}(\hat{X}(S))\leq \hat{L}(M)$, therefore
$$E_S \left[L(\hat{X}(S))-\hat{L}(M)\right]\leq E_S \left[L(\hat{X}(S))-\hat{L}(\hat{X}(S))\right]\leq 2\RR_s(\X)\;.$$
Finally, note that since $M$ has a fixed value, and does not depend on $S$, the empirical reconstruction error $\hat{L}(M)$ is an unbiased estimator of $L(M)$, that is,
$$E_S\left[\hat{L}(M)\right]=L(M)=E_S\left[L(M)\right]\;.$$
Therefore,
$$E_S \left[L(\hat{X}(S))\right]\leq L(M)+2 \RR_s(\H)\;.$$
\end{proof}

In the $\ell_2$ case, where $L(X)=\frac{1}{nm}|Y-X|^2_2$, \citet{SST} derive a bound, which holds in high probability over $S$, in the following form (which we write with the notation of matrices, but is derived generally for any prediction problem):
$$\sup_{X\in\X} L(X)-\hat{L}(X)-\sqrt{A_1\cdot\hat{L}(X)}\leq A_2\;,$$
for some $A_1,A_2$ which depend on the Rademacher complexity of $\X$ (and on $s$ and other parameters). From this, using arguments similar to those used in Lemma~\ref{lem:BM_radbound} above, they derive the bound on $L(\hat{X}(S))$ that is shown above in~(\ref{eq:smoothbound}).

\end{document}